\documentclass[letterpaper, 10 pt, journal, twoside]{IEEEtran} 

\usepackage{savesym}
\usepackage{amsmath,amssymb,amsfonts,amsthm}
\usepackage{algorithm}
\usepackage[noend]{algpseudocode}
\usepackage{comment}
\usepackage{xspace}
\usepackage{scalerel}
\usepackage{xcolor}
\usepackage[caption=false,font=footnotesize]{subfig}
\usepackage{multicol}
\usepackage[bookmarks=true]{hyperref}
\usepackage{bm}
\usepackage{enumerate}
\usepackage{booktabs} 

\usepackage[normalem]{ulem}

\IEEEoverridecommandlockouts
\usepackage{cite}
\usepackage{graphicx}
\def\BibTeX{{\rm B\kern-.05em{\sc i\kern-.025em b}\kern-.08em
				T\kern-.1667em\lower.7ex\hbox{E}\kern-.125emX}}

\begin{document}

\title{
Probabilistic completeness of RRT for geometric and kinodynamic planning\\ with forward propagation\\
{\LARGE Corrigendum}
}

\author{Michal Kleinbort$^{1}$, Kiril Solovey$^{2}$, Zakary Littlefield$^{3}$, Kostas E. Bekris$^{3}$, and Dan Halperin$^{1}$%
\thanks{
Work by D.H. and M.K. has been supported in part by the Israel Science
Foundation (grant nos.~825/15,1736/19),
by NSF/US-Israel-BSF (grant no.~2019754),
by the Israel Ministry of Science and Technology (grant no.~103129),
by the Blavatnik Computer Science Research Fund, and by
the Yandex Machine Learning Initiative for Machine Learning
at Tel Aviv University.
This paper was prepared while K.S.\ was a Ph.D.\ student in Tel Aviv Univesity, where he was supported by
    the Clore Israel Foundation. Z.L.\ and  K.B.\ were supported by NSF IIS 1617744 and CCF 1330789.} 
\thanks{$^{1}$ M.K.\ and D.H.\ are with the Blavatnik School of Computer Science, Tel-Aviv University, Israel}%
\thanks{$^{2}$ K.S.\ is with the Faculty of Electrical and Computer Engineering, Technion, Israel}%
\thanks{$^{3}$ Z.L.\ and K.B.\ are with the Depratment of Computer Science, Rutgers University, NJ~08854, USA}%
}

\maketitle

\def\frechet{Fr\'echet\xspace}

\newcommand{\cupdot}{\mathbin{\mathaccent\cdot\cup}}

\newcommand{\mtm}{\emph{multi-to-multi}\xspace}
\newcommand{\mts}{\emph{multi-to-single}\xspace}
\newcommand{\sts}{\emph{multi-to-single-restricted}\xspace}
\newcommand{\dtd}{\emph{single-to-single}\xspace}

\newcommand{\cte}{\emph{full-to-edge}\xspace}
\newcommand{\ctc}{\emph{full-to-full}\xspace}
\newcommand{\ete}{\emph{edge-to-edge}\xspace}

\newcommand{\AND}{{\sc and}\xspace}
\newcommand{\OR}{{\sc or}\xspace}

\newcommand{\ignore}[1]{}

\def\vor{\text{Vor}}

\def\P{\mathcal{P}} \def\C{\mathcal{C}} \def\H{\mathcal{H}}
\def\F{\mathcal{F}} \def\U{\mathcal{U}} \def\L{\mathcal{L}}
\def\O{\mathcal{O}} \def\I{\mathcal{I}} \def\S{\mathcal{S}}
\def\G{\mathcal{G}} \def\Q{\mathcal{Q}} \def\I{\mathcal{I}}
\def\T{\mathcal{T}} \def\L{\mathcal{L}} \def\N{\mathcal{N}}
\def\V{\mathcal{V}} \def\B{\mathcal{B}} \def\D{\mathcal{D}}
\def\W{\mathcal{W}} \def\R{\mathcal{R}} \def\M{\mathcal{M}}
\def\X{\mathcal{X}} \def\A{\mathcal{A}} \def\Y{\mathcal{Y}}
\def\L{\mathcal{L}}

\def\dS{\mathbb{S}} \def\dT{\mathbb{T}} \def\dC{\mathbb{C}}
\def\dG{\mathbb{G}} \def\dD{\mathbb{D}} \def\dV{\mathbb{V}}
\def\dH{\mathbb{H}} \def\dN{\mathbb{N}} \def\dE{\mathbb{E}}
\def\dR{\mathbb{R}} \def\dM{\mathbb{M}} \def\dm{\mathbb{m}}
\def\dB{\mathbb{B}} \def\dI{\mathbb{I}} \def\dM{\mathbb{M}}
\def\dZ{\mathbb{Z}}

\def\E{\mathbf{E}} 

\def\eps{\varepsilon}

\def\limn{\lim_{n\rightarrow \infty}}

\def\obs{\mathrm{obs}}
\newcommand{\defeq}{%
  \mathrel{\vbox{\offinterlineskip\ialign{%
    \hfil##\hfil\cr
    $\scriptscriptstyle\triangle$\cr
    $=$\cr
}}}}
\def\Int{\mathrm{Int}}

\def\Reals{\mathbb{R}}
\def\Naturals{\mathbb{N}}
\renewcommand{\leq}{\leqslant}
\renewcommand{\geq}{\geqslant}
\newcommand{\compl}{\mathrm{Compl}}

\newcommand{\sig}{\text{sig}}

\newcommand{\sbs}{sampling-based\xspace}
\newcommand{\mr}{multi-robot\xspace}
\newcommand{\mpl}{motion planning\xspace}
\newcommand{\mrmp}{multi-robot motion planning\xspace}
\newcommand{\sr}{single-robot\xspace}
\newcommand{\cs}{configuration space\xspace}
\newcommand{\conf}{configuration\xspace}
\newcommand{\confs}{configurations\xspace}

\newcommand{\stl}{\textsc{Stl}\xspace}
\newcommand{\boost}{\textsc{Boost}\xspace}
\newcommand{\core}{\textsc{Core}\xspace}
\newcommand{\leda}{\textsc{Leda}\xspace}
\newcommand{\cgal}{\textsc{Cgal}\xspace}
\newcommand{\qt}{\textsc{Qt}\xspace}
\newcommand{\gmp}{\textsc{Gmp}\xspace}

\newcommand{\Cpp}{C\raise.08ex\hbox{\tt ++}\xspace}

\def\concept#1{\textsf{\it #1}}
\def\ccode#1{{\texttt{#1}}}

\newcommand{\ch}{\mathrm{ch}}
\newcommand{\pspace}{{\sc pspace}\xspace}
\newcommand{\threesum}{{\sc 3Sum}\xspace}
\newcommand{\np}{{\sc np}\xspace}
\newcommand{\degree}{\ensuremath{^\circ}}
\newcommand{\argmin}{\operatornamewithlimits{argmin}}

\newcommand{\Gdisk}{\G^\textup{disk}}
\newcommand{\Gbt}{\G^\textup{BT}}
\newcommand{\Gsoft}{\G^\textup{soft}}
\newcommand{\Gnear}{\G^\textup{near}}
\newcommand{\Gembed}{\G^\textup{embed}}

\newcommand{\dist}{\textup{dist}}

\newcommand{\Cfree}{\C_{\textup{free}}}
\newcommand{\Cforb}{\C_{\textup{forb}}}

\newtheorem{lemma}{Lemma}
\newtheorem{theorem}{Theorem}
\newtheorem{corollary}{Corollary}
\newtheorem{claim}{Claim}
\newtheorem{proposition}{Proposition}

\theoremstyle{definition}
\newtheorem{definition}{Definition}
\newtheorem{remark}{Remark}
\theoremstyle{plain}
\newtheorem{observation}{Observation}

\def\len{c_\ell}
\def\bot{c_b}

\def\lenopt{\len^*}
\def\botopt{\bot^*}

\def\Im{\textup{Im}}

\def\rfunc{\left(\frac{\log n}{n}\right)^{1/d}}
\def\rfuncs{\left(\frac{\log n}{n}\right)^{1/d}}
\def\cfunc{\sqrt{\frac{\log n}{\log\log n}}}
\def\rtrs{\gamma\rfunc}
\def\ctrs{2\cfunc}
\def\aconn{\A_\textup{conn}}
\def\abd{\A_\textup{str}}
\def\aspan{\A_\textup{span}}
\def\aopt{\A_\textup{opt}}
\def\ao{\A_\textup{ao}}
\def\acfo{\A_\textup{acfo}}
\def\binomial{\textup{Binomial}}
\def\twin{\textup{twin}}

\def\aas{a.a.s.\xspace}
\def\0{\bm{0}}

\def\distU#1{\|#1\|_{\G_n}^U}
\def\distW#1{\|#1\|_{\G_n}^W}

\def\tooth{\scalerel*{\includegraphics{./../fig/tooth}}{b}}

\makeatletter
\def\thmhead@plain#1#2#3{%
  \thmname{#1}\thmnumber{\@ifnotempty{#1}{ }\@upn{#2}}%
  \thmnote{ {\the\thm@notefont#3}}}
\let\thmhead\thmhead@plain
\makeatother

\def\todo#1{\textcolor{blue}{\textbf{TODO:} #1}}
\def\new#1{\textcolor{magenta}{#1}}
\def\kiril#1{\textcolor{magenta}{\textbf{Kiril:} #1}}
\def\old#1{\textcolor{red}{#1}}
\def\michal#1{\textcolor{red}{\textbf{Michal:} #1}}
\def\zak#1{\textcolor{orange}{\textbf{Zak:} #1}}

\def\removed#1{\textcolor{green}{#1}}

\def\dx{\,\mathrm{d}x}
\def\dy{\,\mathrm{d}y}
\def\drho{\,\mathrm{d}\rho}

\newcommand{\prm}{{\tt PRM}\xspace}
\newcommand{\prmstar}{{\tt PRM}$^*$\xspace}
\newcommand{\rrt}{{\tt RRT}\xspace}
\newcommand{\est}{{\tt EST}\xspace}
\newcommand{\grrt}{{\tt GEOM-RRT}\xspace}
\newcommand{\rrtstar}{{\tt RRT}$^*$\xspace}
\newcommand{\rrg}{{\tt RRG}\xspace}
\newcommand{\btt}{{\tt BTT}\xspace}
\newcommand{\fmt}{{\tt FMT}$^*$\xspace}
\newcommand{\dfmt}{{\tt DFMT}$^*$\xspace}
\newcommand{\dprm}{{\tt DPRM}$^*$\xspace}
\newcommand{\mstar}{{\tt M}$^*$\xspace}
\newcommand{\drrtstar}{{\tt dRRT}$^*$\xspace}
\newcommand{\sst}{{\tt SST}\xspace}
\newcommand{\aorrt}{{\tt AO-RRT}\xspace}


\begin{abstract}
The Rapidly-exploring Random Tree~(RRT) algorithm has been one of the most prevalent and popular motion-planning techniques for two decades now. Surprisingly, in spite of its centrality, there has been an active debate under which conditions RRT is probabilistically complete. We provide two new proofs of probabilistic completeness (PC) of RRT with a reduced set of assumptions. The first one for the purely geometric setting, where we only require that the solution path has a certain clearance from the obstacles. For the kinodynamic case with forward propagation of random controls and duration, we only consider in addition mild Lipschitz-continuity conditions. These proofs fill a gap in the study of RRT itself. They also lay sound foundations for a variety of more recent and alternative sampling-based methods, whose PC property relies on that of RRT.

{Our original publication~\cite{KSLBH19} contains an error in the analysis of the case of the kinodynamic RRT. Here, we rectify the problem by modifying the proof of Theorem~\ref{kino_main_thm}, which, in particular, necessitated a revision of Lemma~\ref{fig:prop_lemma}.
Briefly, the original (and erroneous) proof of Theorem~\ref{kino_main_thm} used a sequence of equal-size balls. The correction uses a sequence of balls of increasing radii.
We emphasize that the correction is in Lemma~\ref{fig:prop_lemma} and the proof of Theorem~\ref{kino_main_thm} only. The main results remain unchanged.}
\end{abstract}

\begin{IEEEkeywords}
Motion and Path Planning, Nonholonomic Motion Planning
\end{IEEEkeywords}

\section{Introduction}
\IEEEPARstart{T}wo decades ago LaValle and Kuffner presented the \emph{Rapidly-exploring Random Tree} (\rrt)~\cite{LaVKuf01} method for sampling-based motion planning. Even though numerous alternatives for motion planning have been proposed since then, \rrt remains one of the most widely used techniques today.  This is due to its simplicity and practical efficiency, especially when combined with simple heuristics.

RRT is especially useful in single-query settings, as it focuses on finding a single trajectory moving a robot from an initial state to a goal state (or region), rather than exploring the full state space of the problem, as roadmap methods do, such as \prm~\cite{KavETAL96}. To achieve this objective, \rrt grows a tree, rooted at an initial state, which is periodically extended towards random state samples until the goal is reached.

Notably, \rrt is well suited to complex motion planning tasks and, in particular, problems involving kinodynamic constraints. This is due to the fact that \rrt can be implemented
without a steering function, which is difficult to obtain for many systems with complex dynamics. (This function returns a path between two states in the absence of obstacles. It corresponds to solving a two-point boundary value problem (BVP), which may be a difficult task for many dynamical systems.) Moreover, \rrt has low dependence on parameters and is easily extendable to a variety of domains (e.g., {\tt graspRRT} for integrated motion and grasp planning~\cite{VDAD10}).

Since its introduction, numerous variations and extensions of \rrt have been proposed (see, e.g.,~\cite{JailletCS10,YershovaJSL05,ZuckerETAL07,WangETAL10,KL2000_geom}), to allow improved performance.  While \rrt is not asymptotically optimal (AO) and provably does not converge to the optimal solution \cite{NecETAL10,KF11}, it forms the basis of many AO planners, including \rrtstar and \rrg~\cite{KF11}. In particular, the probabilistic completeness (PC) of  most of the aforementioned \rrt-based algorithms is derived from the PC properties of \rrt. 

Surprisingly, it is not completely obvious under what conditions \rrt is probabilistically complete, especially when using forward propagation of controls for the kinodynamic case. Indeed there has been some debate on this issue in the literature~\cite{KunzS14,CaronETAL17}. This paper aims to address this gap. 

\subsection{Contribution}
We provide two new proofs of PC of \rrt. The first one for the purely geometric setting, where we only require that the solution path has a certain clearance from the obstacles. For the kinodynamic case with forward propagation of random controls and duration, we add mild Lipschitz-continuity conditions. This line of work lays sound foundations for arguing the probabilistic completeness of the variety of methods whose PC relies on that of  \rrt.

Section~\ref{sec:related_work} describes related work and  Section~\ref{sec:geometric_proof} proceeds with the probabilistic completeness proof for the geometric case. Section~\ref{sec:kinodyn} gives a proof for the kinodynamic setting.
A discussion on further research appears in Section~\ref{sec:future}.

\section{Related work}
\label{sec:related_work}
Sampling-based algorithms are among the state-of-the-art alternatives for robot motion planning. Since their introduction in the mid 90's (e.g., \prm, \est~\cite{HLM97} and \rrt), they have been used in numerous robotic tasks. Sampling-based motion planners are also widely used in various fields other than robotics, such as computational biology and digital animation. There are recent reviews that provide a comprehensive coverage of developments in sampling-based motion planning~\cite{ES14, CRCbookChap51}.

Sampling-based planners can potentially provide the following two desirable properties; (i) \emph{probabilistic completeness (PC)} and (ii) \emph{Asymptotic (near)-optimality (AO)}. The former implies that the probability that the planner will return a solution (if one exists) approaches one as the number of samples tends to infinity. AO is a stronger property, as it implies that the cost of the solution returned (if one exists) by the planning algorithm (nearly) approaches the cost of the optimal solution as the number of samples tends to infinity.

AO variants of \rrt and \prm, i.e., the \rrtstar and \prmstar methods, have been introduced more recently~\cite{KF11}.  The same line of work introduced another AO planning algorithm, \rrg, which constructs a connected \prm-like roadmap in a single-query setting. Interestingly, the PC property of both \rrtstar and \rrg relies entirely on the PC property of \rrt. Since then, many variants of \rrtstar and \rrg have been devised~\cite{GamETAL18, SalH16, ArslanT13, Naderi15, OtteF16, DevETAL16}, most of which inherit their PC and AO properties from \rrg and  \rrtstar. A different series of planners implicitly maintain a \prm structure to guarantee AO planning~\cite{JSCP15,ManETAL18,SolHal17,GamETAL15}. A recent paper develops precise conditions for \prm-based planners (in terms of the connection radius used) to guarantee AO~\cite{SK18}.

Although \rrtstar, \prmstar, and their extensions, were initially developed to deal with geometric planning, they can be extended to kinodynamic planning. This requires proper adjustments to the algorithms and the proofs~(see, e.g.,~\cite{SJP_ICRA15,SJP15,KarFra10, KarFra13,WebBer13,PerETAL12,XieETAL15,GorETAL13}). Nevertheless, these approaches require the use of a steering function, which limits their application to systems for which such a function is readily available. Recent work proposes a different type of approach, called \sst, that employs only forward propagation~\cite{LiETAL16} and achieves asymptotic near-optimality. Hauser and Zhou propose a simple yet effective approach termed \aorrt, which employs a forward-propagating \rrt as a black-box component~\cite{HauserZ16}, to achieve AO.

\subsection{PC of Kinodynamic \rrt}
LaValle and Kuffner discuss completeness of \rrt in kinodynamic setting in one of the early works on the subject~\cite{LaVKuf01}. While this work provides strong evidence for the PC of \rrt, it only derives a proof sketch that does not fully addresses many of the complications that arise in analyzing sampling-based planners, be it a geometric~\cite{KL2000_geom} or kinodynamic setting. For instance, the proofs in that paper assume the existence of ``attraction sequences'' and ``basin regions'', whose purpose is to lead the growth of the \rrt tree toward the goal. It is not clear, however, whether such regions exist at all and for what types of robotic systems. It is also not clear whether the number of such regions is finite, and whether it is possible to produce samples in such regions with positive probability.  Similar concerns were expressed by Caron et al.~\cite{CaronETAL17}.

Indeed, in 2014, Kunz and Stilman~\cite{KunzS14} showed that one of the variants of \rrt mentioned in the original \rrt paper~\cite{LaVKuf01} is in fact not PC.
{In particular, they consider \rrt which employs a fixed time step (rather than random propagation time which we use here) and a best-control input strategy, which picks the control input that yields the nearest state to the random sample. For this setting they describe a counterexample consisting of a specific robotic system for which \rrt will have a success rate of~$0$. The reason being that the state space reachable by this type of \rrt is a strict subset of the actual reachable space of the robotic system.} Completeness of the other variants was left as an open question. 

PC proofs of \rrt under different steering functions and robot systems were presented in~\cite{CaronETAL17} and~\cite{KunzS14b}. 
{Specifically, Caron et al.~\cite{CaronETAL17} consider state-based steering, which is different than forward propagation of random controls that we consider here. }
A setting similar to ours of random forward propagation was considered in~\cite{LiETAL16} and~\cite{PapaETAL14}. It should be noted, however, that both papers consider a random-tree planner (and its extensions), which selects the next vertex to expand in a uniform and random manner among all its vertices, unlike \rrt which expands the nearest neighbor toward a random sample point. Interestingly, the random tree is AO, in contrast to \rrt which is not AO~\cite{KF11, NecETAL10}. Nevertheless, the selection process employed by \rrt allows it to quickly explore the underlying state space when endowed with an appropriate metric.

\section{Probabilistic completeness of \rrt: The geometric case}
\label{sec:geometric_proof}
We start by defining useful notation in Subsection~\ref{ssec:geom_prelim} and then proceed to describe \rrt for the geometric case. Then, in Subsection~\ref{ssec:geom_pc_proof}, we provide the PC proof.  We call the algorithm in this section \grrt to distinguish from the kinodynamic version. The geometric case, where a steering function exists and the dimension of the control space is identical to the dimension of the state space, can be considered as a special case of the kinodynamic setting. Thus, this section can be viewed as an introduction to the more involved kinodynamic setting, which is analyzed in the following section. 

\subsection{Preliminaries}
\label{ssec:geom_prelim}
Let $\X$ be the state space, which is assumed to be $[0,1]^d$ (a $d$-dimensional Euclidean hypercube), equipped with the standard Euclidean distance metric, whose norm we denote by $\|\cdot\|$. The free space is denoted by $\F\subseteq\X$.
Given a subset $D\subseteq \X$ we denote by $|D|$ its Lebesgue measure.
We will use $\B_r(x)$ to denote the ball 
of radius $r$ centered at $x\in \dR^d$.
Let $x_{\text{init}}\in \F$ denote the start state, and let $\X_{\text{goal}}$ be an open subset of $\F$
denoting the goal region. For simplicity, we assume that there exist $\delta_{\text{goal}}>0,x_{\text{goal}}\in \X_{\text{goal}}$, such that $\X_{\text{goal}}=\B_{\delta_{\text{goal}}}(x_{\text{goal}})$.

A motion-planning problem is implicitly defined by the triplet $(\F,x_{\text{init}},\X_{\text{goal}})$. A solution to such a problem is a trajectory that moves the robot from the initial state to the goal region while avoiding collisions with obstacles. More formally, a valid trajectory is a continuous map $\pi : [0,t_\pi] \rightarrow \F$, such that $\pi(0) = x_{\text{init}}$ and $\pi(t_\pi)\in \X_{\text{goal}}$. The clearance of $\pi$ is the maximal $\delta_{\text{clear}}$, such that $\B_{\delta_{\text{clear}}}(\pi(t))\subseteq\F$ for all $t\in [0,t_\pi]$. We require that $\delta_{\text{clear}}>0$.

We describe in Algorithm~\ref{algRRTextend} the (geometric) \rrt algorithm, \grrt, based on~\cite{KL2000_geom}.
The input for \grrt consists of an initial configuration $x_{\text{init}}$, goal region $\X_{\text{goal}}$, number of iterations $k$, and a steering parameter $\eta>0$ used by the algorithm.  \grrt constructs a tree $\T$ by preforming $k$ iterations of the following form. In each iteration, a new random sample $x_{\text{rand}}$ is returned from $\X$ uniformly by calling RANDOM\_STATE. Then, the vertex $x_{\text{near}}\in\T$ that is nearest (according to $\|\cdot\|$) to $x_{\text{rand}}$ is found using NEAREST\_NEIGHBOR. A new configuration $x_{\text{new}}\in\X$ is then returned by NEW\_STATE, such that $x_{\text{new}}$ is on the line segment between $x_{\text{near}}$ and $x_{\text{rand}}$ and the distance $\Vert x_{\text{near}} - x_{\text{new}}\Vert$ is at most $\eta$. Finally, COLLISION\_FREE($x_{\text{near}}, x_{\text{new}}$) checks whether the path from $x_{\text{near}}$ to $x_{\text{new}}$ is collision free. If so, $x_{\text{new}}$ is added as a vertex to $\T$ and is connected by an edge from $x_{\text{near}}$. 

\begin{algorithm}
    \caption{\tt GEOM-RRT($x_{\text{init}}, \X_{\text{goal}}, k, \eta$)}
	\label{algRRTextend}
	\begin{algorithmic}[1]
		\State{$\mathcal{T}.\text{init}(x_{\text{init}})$}
		\For {$i = 1 \text{ to } k$}
		\State $x_{\text{rand}}\gets$ RANDOM\_STATE()
		\State $x_{\text{near}}\gets \text{NEAREST\_NEIGHBOR}(x_{\text{rand}}, \mathcal{T} )$
		\State  $x_{\text{new}} \gets$ NEW\_STATE($x_{\text{rand}}, x_{\text{near}}, \eta$)
		\If {COLLISION\_FREE($x_{\text{near}}, x_{\text{new}}$)} 
		\State{$\mathcal{T}.$add\_vertex($x_{\text{new}}$)}
		\State{$\mathcal{T}.$add\_edge($x_{\text{near}}, x_{\text{new}}$)}
		\EndIf
		\EndFor	
		\State 		\Return $\mathcal{T}$
	\end{algorithmic}
\end{algorithm}
To retrieve a trajectory for the robot, the single path in $\T$ from the root state $x_{\text{init}}$ to the goal is found.  It can then be translated to a feasible, collision-free trajectory for the robot by tracing the configurations along this path. 

\subsection{Probabilistic completeness proof}
\label{ssec:geom_pc_proof}
Next we devise a PC proof for \grrt. Throughout this section we will assume that there exists a valid trajectory $\pi:[0,t_{\pi}]\rightarrow \F$ with clearance $\delta_{\text{clear}}>0$. Without loss of generality, assume that $\pi(t_{\pi})=x_{\text{goal}}$, i.e., the trajectory terminates at the center of the goal region. Denote by~$L$ the (Euclidean) length of  $\pi$. Also, let $\delta:=\min\{\delta_{\text{clear}},\delta_{\text{goal}}\}$.

Let $m=\frac{5L}{\nu}$, where $\nu = \min(\delta,\eta)$, and $\eta$ is the steering parameter of \grrt. Then, define a sequence of $m+1$ points $x_0=x_{\text{init}},\ldots,x_{m}=x_{\text{goal}}$ along $\pi$, such that the length of the sub-path between every two consecutive points is $\nu/5$. Therefore, $\Vert x_i-x_{i+1}\Vert\leq\nu/5$ for every $0\leq i< m$. Next, we define a set of $m+1$ balls of radius $\nu/5$, centered at these points, and prove that with high probability \grrt will generate a path that goes through these balls.

We start by proving Lemma~\ref{lem:geometric_step}, which will be used in the proof of Theorem~\ref{thm_main} and specifies a condition for successfully extending the tree to the goal. 

	\begin{figure}
	    \includegraphics[width=0.9\columnwidth]{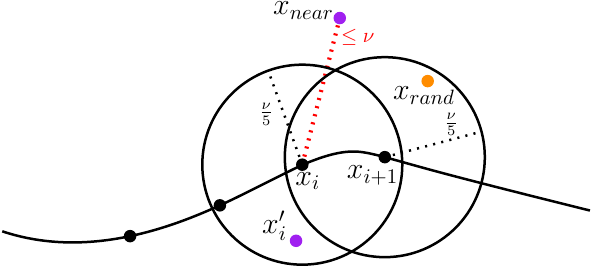}
		\caption{\textsf{Illustration of the proof of Lemma~\ref{lem:geometric_step}.}}
		\label{fig:geom_step}
	\end{figure}

\begin{lemma}
Suppose that \grrt  has reached $\B_{\nu/5}(x_i)$, that is, $\T$ contains a vertex $x'_i$ such that 	$x'_{i}\in \B_{\nu/5}(x_i)$. If a new sample $x_{\text{rand}}$ is drawn such that $x_{\text{rand}}\in \B_{\nu/5}(x_{i+1})$, then the straight line segment between $x_{\text{rand}}$ and its nearest neighbor $x_{\text{near}}$ in $\T$ lies entirely in $\F$.	
\label{lem:geometric_step}
\end{lemma}

\begin{proof}
Denote by $x_{\text{near}}$ the nearest neighbor of $x_{\text{rand}}$ among the \rrt vertices. See Figure~\ref{fig:geom_step} for an illustration. Then, from the definition of $x_{\text{near}}$, it follows that $\Vert x_{\text{near}} - x_{\text{rand}}\Vert \leq \Vert  x'_i - x_{\text{rand}} \Vert$, where $x'_{i}\in \B_{\nu/5}(x_i)$.

We show that $x_{\text{near}}$ must lie in $\B_\nu(x_i)$, implying that $\overline{x_{\text{near}}x_{\text{rand}}}\subset\F$, as $x_{rand}\in \B_{\nu/5}(x_{i+1}) \subset \B_{\nu}(x_{i})$. From $\Vert x_{\text{near}} - x_{\text{rand}}\Vert \leq \Vert  x'_i - x_{\text{rand}} \Vert$ and the triangle inequality, we have:
\begin{align*}
	\Vert  x_{\text{near}} - x_i \Vert &\leq 
	\Vert  x_{\text{near}}- x_{\text{rand}} \Vert + \Vert  x_{\text{rand}} - x_i \Vert 
	\\
	&\leq 
	\Vert  x'_i - x_{\text{rand}} \Vert + \Vert  x_{\text{rand}} - x_i \Vert .
	\end{align*}
	From the triangle inequality, we have that 
	\[\Vert x_{\text{rand}} - x_i\Vert \leq \Vert  x_{\text{rand}} - x_{i+1} \Vert + \Vert  x_{i+1} - x_{i} \Vert,\]
	\[\Vert  x'_i - x_{\text{rand}} \Vert \leq \Vert  x'_i - x_{i} \Vert + \Vert  x_{i} - x_{i+1} \Vert +
		\Vert  x_{i+1} - x_{\text{rand}} \Vert.
	\]
	Therefore:
	\begin{align*}
	\Vert  x_{\text{near}} - x_i \Vert &\leq 
	\Vert x'_i - x_i \Vert + 2\Vert x_{i+1} - x_{\text{rand}}\Vert + \\
	& 2\Vert x_{i+1}-x_{i}\Vert \leq 5\frac{\nu}{5} = \nu.
	\end{align*}
	Hence, $x_{\text{near}}\in\B_\nu(x_{i})\subseteq\F$ and thus $\overline{x_{\text{near}}x_{\text{rand}}}\subset\F$. 
	
	Note that $\Vert x_{\text{near}} - x_{\text{rand}} \Vert \leq \eta$, since: 
	$\Vert x_{\text{rand}} - x_{\text{near}} \Vert \leq \Vert x_{\text{rand}} - x'_i \Vert
\leq \Vert x_i' - x_i \Vert +  \Vert x_i - x_{i+1} \Vert  + \Vert x_{i+1} - x_{\text{rand}} \Vert \leq 3\cdot\frac{\nu}{5} < \nu \leq \eta.$ 
The fact that $\Vert x_{\text{near}} - x_{\text{rand}} \Vert \leq \eta$, means that $x_{\text{new}} = x_{\text{rand}}$.

\end{proof}

We now prove our main theorem.
\begin{theorem}
	The probability that \grrt fails to reach $\X_{\text{goal}}$ from $x_{\text{init}}$ after $k$ iterations
	is at most $ae^{-bk}$, for some constants $a,b \in \dR_{>0}$.
	\label{thm_main}
\end{theorem}

\begin{proof}

Assume that $\B_{\nu/5}(x_{i})$ already contains an \rrt vertex.  Let $p$ be the probability that in the next iteration an \rrt vertex will be added to $\B_{\nu/5}(x_{i+1})$. Recall that due to Lemma~\ref{lem:geometric_step}, $x_{\text{rand}}\in \B_{\nu/5}(x_{i+1})$ ensures that \rrt will reach $\B_{\nu/5}(x_{i+1})$.  Since at each iteration $i$ we draw $x_{\text{rand}}$ uniformly at random from~$[0,1]^d$, the probability $p$ that this sample falls inside $\B_{\nu/5}(x_{i+1})$ is equal to $|\B_{\nu/5}|/|[0,1]^d|=|\B_{\nu/5}|$.

	\begin{figure}[h]
	    \includegraphics[width=\columnwidth]{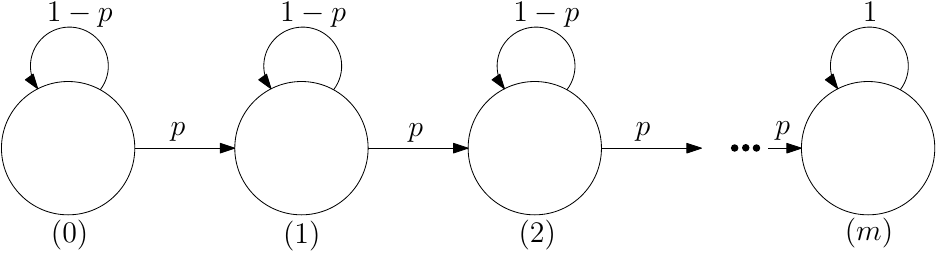}
		\caption{\textsf{A Markov chain where the success probability $p= |\B_{\nu/5}|$
				is the probability to uniformly sample from a specific ball of radius $\nu/5$.
				State $(m)$ is a terminal state.
				$m$ successful outcomes imply that the algorithm finds a path from initial state to goal, where the $i$th successful
				outcome switches from state $i$ to state $i+1$.}}
		\label{fig:markov}
	\end{figure}

In order for  \grrt to reach $\X_{\text{goal}}$ from $x_{\text{init}}$ we need to repeat this step $m$ times from $x_i$ to $x_{i+1}$ for $0\leq i < m$. This stochastic process can be viewed as a Markov chain (see Figure~\ref{fig:markov}).
Alternatively, this process can be described as $k$ Bernoulli trials with success probability $p$. The  planning problem can be solved after $m$ successful outcomes (the $i$th outcome adds an \rrt vertex in $\B_{\nu/5}(x_i)$). Note that it is possible that the process ends after less than $m$ successful outcomes, i.e., by defining success to be $m$ successful outcomes we obtain an upper bound on the probability of failure.

	Next, we bound the probability of failure, that is, the probability that the process does not reach state $(m)$,
	after $k$ steps. Let $X_k$ denote the number of successes in $k$ trials, then
	\begin{align}
		\Pr[&X_k < m] = \sum_{i=0}^{m-1}{\binom{k}{i}p^i(1-p)^{k-i}}\nonumber\\
		&\leq \sum_{i=0}^{m-1}{\binom{k}{m-1}p^i(1-p)^{k-i}}\nonumber \\
		&\leq \binom{k}{m-1}\sum_{i=0}^{m-1}{(1-p)^{k}} \nonumber \\
		&\leq \binom{k}{m-1}\sum_{i=0}^{m-1}{(e^{-p})^{k}}  
		= \binom{k}{m-1}m{e^{-pk}} \nonumber \\
		&= \frac{\prod_{i=k-m}^k{i}}{(k-1)!}{me^{-pk}}  
		\leq \frac{m}{(m-1)!}k^m{e^{-pk}}\nonumber,  
	\end{align}
	{where the transitions rely on (i) $m\ll k$, (ii)  $p<\frac{1}{2}$,
	and (iii)  $(1-p)\leq e^{-p}$.}

	As $p,m$ are fixed and independent of $k$, the expression $\frac{1}{(m-1)!}k^m m{e^{-pk}}$ decays to zero exponentially with $k$.	Therefore, \grrt with uniform samples is probabilistically complete.
\end{proof}

\section{Probabilistic completeness of \rrt under differential constraints}
\label{sec:kinodyn}
We begin by formulating the kinodynamic problem. Our assumptions on the robotic system and the environment as well as the definitions appear in Subsection~\ref{ssec:dyn_prelim} and are adapted from Li et al.~\cite{LiETAL16}.
Next, we describe the  modifications to \rrt required for solving the kinodynamic problem.  Finally, in Subsection~\ref{ssec:dyn_proof}, we devise a novel PC proof for the kinodynamic \rrt.

\subsection{Preliminaries}
\label{ssec:dyn_prelim}
We adapt the problem attributes introduced in the previous section to accommodate the more involved structure of the kinodynamic case. The state space $\X\subseteq \dR^d$ is a smooth $d$-dimensional manifold. Let $\F\subset \X$ denote the free state space.
As before, we assume that there exist $x_{\text{goal}}\in \X, \delta_{\text{goal}}>0$, such that $\X_{\text{goal}}=\B_{\delta_{\text{goal}}}(x_{\text{goal}})$.

Let $\mathbb{U}\subseteq\dR^D$ denote the space of control vectors. The given system has differential constraints of the following form:
\begin{equation}
\dot{x}(t)= f(x(t), u(t)),\quad x(t)\in \X,\quad  u(t)\in \mathbb{U}.
\label{eq:dyn}
\end{equation} 

Trajectories under differential constraints are defined as follows.
\begin{definition}[]
	A valid trajectory $\pi$  of duration $t_\pi$ is a continuous function
	$\pi: [0, t_\pi ] \rightarrow \F$. A trajectory
	$\pi$ is generated by starting at a given state $\pi(0)$ and
	applying a control function $\Upsilon : [0,t_\pi ] \rightarrow \mathbb{U}$ by forward
	integrating Equation~\ref{eq:dyn}.
\end{definition}
\noindent Similar to prior work~\cite{LiETAL16}, we consider control functions that are  piecewise constant:

\begin{definition}[]
	A piecewise constant control function $\overline{\Upsilon}$  with resolution $\Delta t$ is the concatenation of constant control functions $\bar{\Upsilon}_i : [0, \Delta t ] \rightarrow u_i$, where $u_i\in \mathbb{U}$, and $1\leq i\leq k$, for some $k\in \mathbb{N}_{>0}$. 
\end{definition}
We assume that the system is Lipschitz continuous for both of its arguments. That is,
$\exists K_u, K_x > 0 $ s.t. $\forall\ x_0,x_1\in \X,u_0,u_1\in \mathbb{U}$:
\[\lVert f(x_0, u_0)- f(x_0, u_1)\rVert \leq K_u\lVert u_0-u_1\rVert,\]
\[\lVert f(x_0, u_0)- f(x_1, u_0)\rVert \leq K_x\lVert x_0-x_1\rVert.\]

We describe here the (kinodynamic) \rrt algorithm, based on~\cite{LaVKuf01}. 
\begin{algorithm}
	\caption{\tt RRT($x_{\text{init}}, \X_{\text{goal}}, k, T_{\text{prop}}, \mathbb{U}$)}
	\label{algRRTextendKINO}
	\begin{algorithmic}[1]
		\State{$\mathcal{T}.\text{init}(x_{\text{init}})$}
		\For {$i = 1 \text{ to } k$}
		\State $x_{\text{rand}}\gets$ RANDOM\_STATE()
		\State $x_{\text{near}}\gets \text{NEAREST\_NEIGHBOR}(x_{\text{rand}}, \mathcal{T} )$
		\State $t \gets $ SAMPLE\_DURATION($0, T_{\text{prop}}$)
		\State $u \gets $ SAMPLE\_CONTROL\_INPUT($\mathbb{U}$)
		\State  $x_{\text{new}} \gets$ PROPAGATE($x_{\text{near}}, u, t$)
		\If {COLLISION\_FREE($x_{\text{near}}, x_{\text{new}}$)} 
		\State{$\mathcal{T}.$add\_vertex($x_{\text{new}}$)}
		\State{$\mathcal{T}.$add\_edge($x_{\text{near}}, x_{\text{new}}$)}
		\EndIf
		\EndFor	
		\State 	\Return $\mathcal{T}$
	\end{algorithmic}
\end{algorithm}

The \rrt algorithm in dynamic settings with no BVP solver has the following 
inputs: start state $x_{\text{init}}$, goal region $\X_{\text{goal}}$, the number of iterations $k$, the maximal time duration for propagation $T_{\text{prop}}$, and the set of control inputs $\mathbb{U}$. Our proof below assumes that $T_{\text{prop}}$ is positive and independent of $k$.
 
Lines~5--7 in Algorithm~\ref{algRRTextendKINO} replace line~5 in Algorithm~\ref{algRRTextend}.
Here, a random time duration $t$ is chosen between $0$ and $T_{\text{prop}}$ as well as a random control input $u\in \mathbb{U}$.
The algorithm uses a forward propagation approach (function PROPAGATE) from $x_{\text{near}}$: control input $u$ is applied for time duration $t$, reaching a new state $x_{\text{new}}$. Finally, if the trajectory from $x_{\text{near}}$ to $x_{\text{new}}$ is collision-free, then $x_{\text{new}}$ is added to $\T$ together with a connecting edge to $x_{\text{near}}$.

\subsection{Probabilistic completeness proof}
\label{ssec:dyn_proof}
We prove that \rrt for a system with dynamics satisfying the aforementioned characteristics is PC. To do so, we start by proving three lemmas. 
The following lemma, which is an extension of Theorem~15 from~\cite{LiETAL16}, bounds the distance between the endpoints of two trajectories with similar control inputs and initial positions, for the same duration. 
\begin{lemma}
Let $\pi,\pi'$ be two trajectories, with the corresponding control functions $\Upsilon(t),\Upsilon'(t)$. Suppose that $x_0 = \pi(0), x_0' = \pi'(0)$. 
Let $T>0$ be a time duration such that for all $t\in[0,T]$ it holds that 
$\Upsilon(t) = u, \Upsilon'(t) = u'$. That is, $\Upsilon, \Upsilon'$ remain fixed throughout $[0,T]$. Then 
\[ \|\pi(T) -\pi'(T)\|\leq e^{K_x T} \Delta x + K_u Te^{K_x T}\Delta u,\]
where $\Delta x = \Vert x_0 - x_0'\Vert$ and $\Delta u = \Vert u - u'\Vert$. 
\label{lem:bound_on_endpoints}
\end{lemma}

\begin{proof}
From the Lipschitz continuity assumption and the triangle inequality, we have that 
\[\Vert f(x_0, u) - f(x'_0, u')\Vert \leq K_u \Delta u  + K_x \Delta x. \]
As in the proof of Theorem~15 in~\cite{LiETAL16}, we will use the Euler integration method to approximate the value of the trajectory $\pi$ at duration $T$.
We divide $[0,T]$ into $\ell\in \dN_{>0}$ pieces, each of duration $h$, i.e., $T = \ell\cdot h$.
Let $x_i, x_i'$ denote the resulting approximations of the trajectories $\pi, \pi'$ at duration $i\cdot h$. From Euler's method we have that
\[x_i = x_{i-1} + h\cdot f(x_{i-1}, u),\]
\[x_i' = x_{i-1}' + h\cdot f(x_{i-1}', u').\]
The proof in~\cite{LiETAL16} shows that 
\begin{equation}
\Vert x_\ell - x_\ell' \Vert < (1 +K_x h)^\ell \Delta x + K_u Te^{K_x T}\Delta u.
\end{equation}
Since 
$(1+K_x h)^\ell = (1+K_x T/\ell)^\ell < e^{K_x T}$ 
we have that 
\[ \Vert x_\ell - x_\ell' \Vert < e^{K_x T}\Delta x + K_u Te^{K_x T}\Delta u.\]

From the Lipschitz continuity assumption we have that the Euler integration method converges to the solution of the \emph{Initial value problem}. That is, 
$\forall 0<i\leq\ell$,
\[\lim_{\ell\to\infty,\: h\to 0,\: \ell h = T} {\|\pi(i\cdot h) -  x_i\|} = 0, \]
\[\lim_{\ell\to\infty,\: h\to 0,\: \ell h = T} {\|\pi'(i\cdot h) -  x_i'\|} = 0.\]
Therefore,
\[\Vert \pi(T)-  \pi'(T)\Vert \leq e^{K_x T} \Delta x + K_u Te^{K_x T}\Delta u.\]
\end{proof}

Next, we give a lower bound on the probability of a successful forward propagation step of \rrt (Algorithm~\ref{algRRTextendKINO}), from a given tree node, using a random control $u\in \mathbb{U}$ and a random duration $t\in T_{\text{prop}}$.
We note that our proof uses a construction similar to~\cite[proof of Theorem~17]{LiETAL16}.
\begin{lemma}
	 Let $\pi$ be a trajectory with clearance $\delta >0$, and  duration $\tau\leq T_{\text{prop}}$.
	 Suppose that the control function $\Upsilon$ is fixed  for all $t\in[0,\tau]$, i.e., $\Upsilon(t)=u\in \mathbb{U}$. 
	 Denote by $x_i,x_{i+1}$ the states  $\pi(0), \pi(\tau)$, respectively.
	 {Let $r_i, r_{i+1}\in \dR_{>0}$ , such that $r_{i+1} = 4e^{K_x \tau}\cdot r_i$ and $r_{i+1}\leq \delta$.}

	 Suppose that the propagation step begins at state $x_i' \in \B_{r_i}(x_i)$ and ends in $x_{i+1}'$. Then for any $\kappa\in (0,1], \epsilon_i \in (0,\kappa{r_{i+1}})$, we have that:
	
	    \[\rho_i:=\Pr[x'_{i+1}\in \B_{\kappa{r_{i+1}}} (x_{i+1})]\geq 
	    p_t\cdot  \frac{\zeta_{D}\cdot \max\left(\frac{(4\kappa-1)e^{K_x \tau} r_{i} - \epsilon_i}{K_u \tau e^{K_x \tau}},0\right)}{|\mathbb{U}|}, \]
	    where $\zeta_D$ is the Lebesgue measure of the unit ball in $\mathbb{R}^D$ and $0<p_t\leq 1$ is some constant. 
	  \label{lem:prop_bound}
\end{lemma}

\begin{proof}
Consider a sequence of balls of radius $r' = \kappa r_{i+1}-\epsilon_i$,
such that (i) the center $c_t$ of each ball lies on $\pi$, that is, $c_t = \pi(t)$ for some duration $t\in [0,\tau]$, and (ii) $\B_{r'}(c_t)\subset \B_{\kappa r_{i+1}}(x_{i+1})$.
The centers of all such balls constitute a segment of the trajectory $\pi$ whose duration is $T_\kappa$. See Figure~\ref{fig:prop_lemma} for an illustration.
\begin{figure}
        \centering
	    \includegraphics[width=\columnwidth]{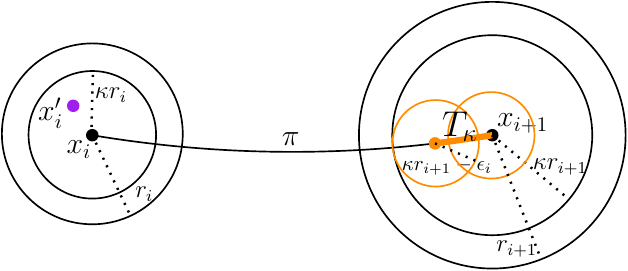}
		\caption{\textsf{Illustration of $T_\kappa$.
		  }}
		\label{fig:prop_lemma}
		\vspace{-10pt}
	\end{figure}

Fix $t\in [0,\tau]$, such that $\B_{r'}(c_t)\subset \B_{\kappa r_{i+1}}(x_{i+1})$. Additionally denote by $u_{\text{rand}}$ the random control generated by \rrt, and denote by $\pi_t$ the trajectory corresponding to the propagation step starting at $x'_i$, using the control $u_{\text{rand}}$ and duration $t$. By Lemma~\ref{lem:bound_on_endpoints}, we have that:
\[ \Vert\pi(t) -  \pi_t(t)\Vert < e^{K_x t} r_i + K_u te^{K_x t}\Delta u, \]
where $\Delta u=\|u-u_{\text{rand}}\|$. 
Now, we wish to find the value $\Delta u$ such that $\Vert\pi(t) -\pi_t(t)\Vert < \kappa r_{i+1}-\epsilon_i$, which would imply that $\pi_t(t)=x'_{i+1}\in \B_{\kappa r_{i+1}} (x_{i+1})$. Thus, we require that 
\[ e^{K_x t} r_i + K_u te^{K_x t}\Delta u < \kappa r_{i+1} - \epsilon_i. \]
As $r_{i+1}=4e^{K_x \tau}\cdot r_i$ the above constraint yields the condition 
\[ e^{K_x t} r_i + K_u te^{K_x t}\Delta u < \kappa \cdot 4e^{K_x \tau} r_i - \epsilon_i \]
which implies that 
\[\Delta u < \frac{(4\kappa e^{K_x \tau}-e^{K_x t}) r_i - \epsilon_i}{K_u t e^{K_x t}}.\]
To ensure that the bound holds for all possible durations $t$ in the relevant range, we should consider $t=\tau$, which is the maximal duration there, as the above expression is decreasing with $t$. That is, we enforce the following bound 
\[\Delta u < \frac{(4\kappa e^{K_x \tau}-e^{K_x \tau}) r_i - \epsilon_i}{K_u \tau e^{K_x \tau}} = \frac{(4\kappa-1) e^{K_x \tau} r_i - \epsilon_i}{K_u \tau e^{K_x \tau}}.\]

To summarize, we have shown that for certain values of $t$ and $u_{\text{rand}}$ it is guaranteed to have $x'_{i+1}\in \B_{\kappa r_{i+1}} (x_{i+1})$. 
It remains to calculate the probability of randomly choosing such  values. The probability for successful  propagation is  at least the  (a) probability of choosing a proper $t$ such that $\pi(t)$ is a center $c_t$ of a small ball $\B_{r'}(c_t)\subset  \B_{\kappa r_{i+1}}(x_{i+1})$ 
times the (b) probability for choosing a control input that will cause $\pi_t(t)$  to fall inside $\B_{r'}(c_t)\subset \B_{\kappa r_{i+1}}(x_{i+1})$. 

Clearly, the probability to choose a proper duration for propagation is at least $p_t = T_\kappa/T_{\text{prop}}>0$.
The probability\footnote{The maxima function guarantees that the probability will be valid, that is, at least 0.} to choose a proper control input is at least:
\[p_u=\frac{\zeta_{D}\cdot \max(\frac{(4\kappa-1) e^{K_x \tau}r_i - \epsilon_i}{K_u \tau e^{K_x \tau}},0)}{|\mathbb{U}|}.\]
Therefore, the probability for successfully propagating is at least $\rho_i = p_t\cdot p_u$.
\end{proof}
Finally, we prove a lower bound on the probability to grow the tree from a vertex in a certain ball.
\begin{lemma}
    Let $x\in \dR^d$ be such that $\B_r(x)\subset \F$.
    Suppose that there exists an \rrt vertex $v\in \B_{2r/5}(x)$.
    Let $x_{\text{near}}$ denote the nearest neighbor of $x_{\text{rand}}$ among all \rrt vertices (see Algorithm~\ref{algRRTextendKINO}).
    The probability that $x_{\text{near}}\in \B_{r}(x)$ is at least $|\B_{r/5}|/|\X|$.
    \label{lem:nn_prob}
\end{lemma}
\begin{proof}
    Suppose that there exists an \rrt vertex $z\not\in \B_{r}(x)$, as otherwise it is immediate that  $x_{\text{near}}\in \B_r(x)$. We show that if $x_{\text{rand}}\in \B_{r/5}(x)$ then $x_{\text{near}}\in \B_r(x)$. See Figure~\ref{fig:NN_lemma} for an illustration of the proof.
    \begin{figure}
        \centering
	    \includegraphics[width=0.5\columnwidth]{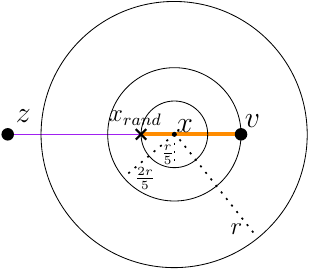}
		\caption{\textsf{Illustration of the proof of Lemma~\ref{lem:nn_prob}.
		        $z,v$ are \rrt vertices. $x_{\text{rand}}$ is the sampled state.
		        Its nearest neighbor will be a vertex in $\B_{r}(x)$.
		  }}
		\label{fig:NN_lemma}
	\end{figure}
    
    Observe that $\|x_{\text{rand}} - v\| \leq 3r/5$ and $\|x_{\text{rand}}-z\|>4r/5$. Thus, $v$ is closer to $x_{\text{rand}}$ than $z$ is, implying that $z$ will not be reported as the nearest neighbor of $x_{\text{rand}}$. 
    If $x_{\text{near}}\neq v$, then there must be another \rrt vertex $y\in \B_{ 3r/5}(x_{\text{rand}})\subset \B_{r}(x)$ such that $\Vert y-x_{\text{rand}}\Vert$ is minimal. 
    Finally, the probability to choose $x_{\text{rand}}\in \B_{r/5}(x)$ is $|\B_{r/5}|/|\X|$.
\end{proof}

Now we are ready to prove our main theorem.
\begin{theorem}
		Suppose that there exists a valid trajectory $\pi$ from $x_{\text{init}}$ to $x_{\text{goal}}$ lying in $\F$, with clearance $\delta_{\text{clear}}>0$.
		Suppose that the trajectory $\pi$ has a piecewise constant control function. 
		Then the probability that \rrt fails to reach $\X_{\text{goal}}$ from $x_{\text{init}}$ after $k$ iterations is at most $a'e^{-b'k}$, for some constants $a',b'\in \dR_{>0}$. 
		\label{kino_main_thm}
\end{theorem}

\begin{proof}
    Let $\tau\leq T_{\text{prop}}$ be a fixed duration for which there exists $\ell\in\dN_{>0}$  such that $\ell\cdot\tau = \Delta t$.
    
    We choose a set of  {times} $t_0=0, t_1, t_2, \ldots, t_m = t_\pi$, such that the difference between every two consecutive ones is $\tau$, where~$t_\pi$ is the duration of $\pi$.
    Let $x_0= \pi(t_0), x_1 = \pi(t_1), \ldots, x_m= \pi(t_m)$ be states along the path $\pi$ that are obtained after duration $t_0,t_1, \ldots , t_m$, respectively. That is, $x_i = \pi(t_i)$.
    Obviously, $m = t_\pi/\tau$ is some constant independent of the number of samples.
    
    We now place a set of $m+1$ balls centered at $x_0, \ldots, x_m$ such that the radius of the $i$th ball is $r_i = (4e^{K_x \tau})^i \cdot r_0$ 
    for $0\leq i\leq m$.
    Requiring that $r_m = \min\{\delta_{\text{goal}},\delta_{\text{clear}}\}$, we obtain a value for the smallest radius $r_0$.
    We show that given that an \rrt vertex in the $i$th ball exists, the probability $p_i$ that in the next iteration \rrt will generate a new vertex in the $(i+1)$st ball when propagating from a vertex in the $i$th ball is bounded from below by a positive constant. More accurately, we show that $p_i\geq p_0$, where $p_0$ is the probability that \rrt will generate a new vertex in $\B_{r_1}(x_1)$ when propagating from $x_0=x_{\text{init}}$ and it is positive.
	The rest of the proof is the same as that of Theorem~\ref{thm_main}.

	
	{Recall that Lemma~\ref{lem:prop_bound} shows a lower bound $\rho_i$ on the probability of a successful propagation between two consecutive balls of radii $r_i, r_{i+1}$ placed in~$x_i=\pi(t_i), x_{i+1} = \pi(t_{i+1})$, respectively, such that $t_{i+1}-t_i = \tau$. } 
	Assign $\kappa$ from Lemma~\ref{lem:prop_bound} the value  $2/5$ and fix $\epsilon_i=\kappa r_0=2r_0/5$ for all $0\leq i\leq m$ (note that $\epsilon_i\in (0,\kappa r_i)$, as required).
	Then $\rho_i>0$ for a duration $\tau$ if 	
	\begin{align}
	  \left(4\cdot\frac{2}{5}-1\right)e^{K_x \tau} r_i - \epsilon_i &= \frac{3}{5}e^{K_x \tau} r_i - \epsilon_i \\ &= \frac{3}{5}e^{K_x \tau} r_i - \frac{2r_0}{5}>0.
	  \label{eq:positive_dist}
	\end{align} 
	If the above expression is satisfied for $i=0$ then it also must hold for $1\leq i\leq m$ as $r_i>r_0$. Since $e^{K_x\tau}\geq 1$ for any $\tau\geq 0$ it must follow that 
	\[\frac{3}{5}e^{K_x \tau} r_0 - \frac{2r_0}{5} \geq \frac{3}{5}r_0 - \frac{2}{5}r_0 = \frac{r_0}{5}>0.\]
	Moreover,  we may set $\tau\leq T_{\text{prop}}$ 
	such that there exists $\ell\in\dN_{>0}$  for which $\ell\cdot\tau = \Delta t$ holds.
    
    Suppose that there exists an \rrt vertex $v\in \B_{2 r_i/5}(x_{i})\subset \B_{r_i}(x_{i})$. 
    We need to bound the probability $p_i$ that in the next iteration the \rrt tree will grow from  an \rrt vertex in $\B_{r_{i}}(x_{i})$, given that an \rrt vertex in $\B_{2{r_i}/5}(x_{i})$ exists, and that the propagation step will add a vertex to $\B_{2{r_{i+1}}/5}(x_{i+1})$. 
    That is, $p_i$ is the probability that in the next iteration both $x_{\text{near}}\in \B_{r_i}(x_{i})$ and $x_{\text{new}}\in \B_{2{r_{i+1}}/5}(x_{i+1})$.
   
    From Lemma~\ref{lem:nn_prob}, we have that the probability $q_i$ that $x_{\text{near}}$ lies in $\B_{r_i}(x_{i})$, given that there exists an \rrt vertex in $\B_{2r_i/5}(x_{i})$, is at least $|\B_{r_i/5}|/|\X|$.
    Now, since $r_i\geq r_0$ for $0\leq i\leq m-1$, we have that $q_i\geq q_0 >0$.
    From Lemma~\ref{lem:prop_bound} we have that 
    the probability for $x_{\text{new}}\in \B_{2r_{i+1}/5}(x_{i+1})$ is at least some positive constant $\rho_i>0$. Moreover, it holds that $\rho_i \geq \rho_0$ for $0\leq i\leq m-1$.
	Hence, for all $0\leq i\leq m-1$  it holds that $p_i\geq p_0$, where $p_0 = q_0\cdot \rho_0 >0$. The rest of the proof is the same as that of Theorem~\ref{thm_main}.
\end{proof}

\section{Discussion}
\label{sec:future}
Although our proofs assume uniform samples, they can be easily extended to samples generated using a Poisson point process, which is preferable in certain settings~\cite{KF11, SK18}. 
An immediate extension of this work
is to verify whether our proofs hold when other sampling distributions are considered, e.g., Halton sequences (see~\cite{JanIP18}).

Another possible direction is to further relax some of the assumptions made for kinodynamic systems, such as Lipschitz continuity. Additionally, the work raises the following challenging research question: Is it possible to extend these proofs that have a reduced set of assumptions to other sampling-based planners~\cite{HLM97}, or informed variants of \rrt.

{Finally, we mention that the following variants of \rrt are not addressed in the current paper, or in the work of Kunz and Stilman~\cite{KunzS14}: (i) random time + best-control input; (ii) fixed time + random control; (iii) random time larger than a fixed threshold + random or best control. Whether these variants are indeed probabilistically complete remains as a question for future research.}

\section*{Acknowledgements}
The authors thank Albert Wu and Thomas Lew for spotting an error in a previous version of the proof of Theorem~2.

\bibliographystyle{IEEEtran}
\bibliography{bibliography}

\begin{thebibliography}{10}
\providecommand{\url}[1]{#1}
\csname url@samestyle\endcsname
\providecommand{\newblock}{\relax}
\providecommand{\bibinfo}[2]{#2}
\providecommand{\BIBentrySTDinterwordspacing}{\spaceskip=0pt\relax}
\providecommand{\BIBentryALTinterwordstretchfactor}{4}
\providecommand{\BIBentryALTinterwordspacing}{\spaceskip=\fontdimen2\font plus
\BIBentryALTinterwordstretchfactor\fontdimen3\font minus
  \fontdimen4\font\relax}
\providecommand{\BIBforeignlanguage}[2]{{%
\expandafter\ifx\csname l@#1\endcsname\relax
\typeout{** WARNING: IEEEtran.bst: No hyphenation pattern has been}%
\typeout{** loaded for the language `#1'. Using the pattern for}%
\typeout{** the default language instead.}%
\else
\language=\csname l@#1\endcsname
\fi
#2}}
\providecommand{\BIBdecl}{\relax}
\BIBdecl

\bibitem{KSLBH19}
M.~Kleinbort, K.~Solovey, Z.~Littlefield, K.~E. Bekris, and D.~Halperin,
  ``Probabilistic completeness of {RRT} for geometric and kinodynamic planning
  with forward propagation,'' \emph{IEEE Robotics and Automation Letters},
  vol.~4, no.~2, pp. x--xvi, 2019.

\bibitem{LaVKuf01}
S.~M. LaValle and J.~J. Kuffner, ``Randomized kinodynamic planning,'' \emph{I.
  J. Robotics Res.}, vol.~20, no.~5, pp. 378--400, 2001.

\bibitem{KavETAL96}
L.~E. Kavraki, P.~Svestka, J.-C. Latombe, and M.~Overmars, ``Probabilistic
  roadmaps for path planning in high dimensional configuration spaces,''
  \emph{{IEEE} Trans. Robotics and Automation}, vol.~12, no.~4, pp. 566--580,
  1996.

\bibitem{VDAD10}
N.~Vahrenkamp, M.~Do, T.~Asfour, and R.~Dillmann, ``Integrated grasp and motion
  planning,'' in \emph{ICRA}, 2010, pp. 2883--2888.

\bibitem{JailletCS10}
L.~Jaillet, J.~Cort{\'e}s, and T.~Sim{\'e}on, ``Sampling-based path planning on
  configuration-space costmaps,'' \emph{IEEE Trans. Robotics}, vol.~26, no.~4,
  pp. 635--646, 2010.

\bibitem{YershovaJSL05}
A.~Yershova, L.~Jaillet, T.~Sim{\'{e}}on, and S.~M. LaValle, ``Dynamic-domain
  {RRT}s: {E}fficient exploration by controlling the sampling domain,'' in
  \emph{ICRA}, 2005, pp. 3856--3861.

\bibitem{ZuckerETAL07}
M.~Zucker, J.~J. Kuffner, and M.~S. Branicky, ``Multipartite {RRTs} for rapid
  replanning in dynamic environments,'' in \emph{{ICRA}}, 2007, pp. 1603--1609.

\bibitem{WangETAL10}
W.~Wang, Y.~Li, X.~Xu, and S.~X. Yang, ``An adaptive roadmap guided
  multi-{RRTs} strategy for single query path planning,'' in \emph{ICRA}, 2010,
  pp. 2871--2876.

\bibitem{KL2000_geom}
J.~J. Kuffner and S.~M. LaValle, ``{RRT}-{C}onnect: An efficient approach to
  single-query path planning,'' in \emph{ICRA}, 2000, pp. 995--1001.

\bibitem{NecETAL10}
O.~Nechushtan, B.~Raveh, and D.~Halperin, ``Sampling-diagram automata: A tool
  for analyzing path quality in tree planners,'' in \emph{WAFR}, 2010, pp.
  285--301.

\bibitem{KF11}
S.~Karaman and E.~Frazzoli, ``Sampling-based algorithms for optimal motion
  planning,'' \emph{I. J. Robotics Res.}, vol.~30, no.~7, pp. 846--894, 2011.

\bibitem{KunzS14}
T.~Kunz and M.~Stilman, ``Kinodynamic {RRTs} with fixed time step and
  best-input extension are not probabilistically complete,'' in \emph{WAFR},
  2014, pp. 233--244.

\bibitem{CaronETAL17}
S.~Caron, Q.~Pham, and Y.~Nakamura, ``Completeness of randomized kinodynamic
  planners with state-based steering,'' \emph{Robotics and Autonomous Systems},
  vol.~89, pp. 85--94, 2017.

\bibitem{HLM97}
D.~Hsu, J.-C. Latombe, and R.~Motwani, ``Path planning in expansive
  configuration spaces,'' in \emph{ICRA}, vol.~3, 1997, pp. 2719--2726.

\bibitem{ES14}
M.~Elbanhawi and M.~Simic, ``Sampling-based robot motion planning: A review,''
  \emph{IEEE Access}, vol.~2, pp. 56--77, 2014.

\bibitem{CRCbookChap51}
D.~Halperin, L.~Kavraki, and K.~Solovey, ``Robotics,'' in \emph{Handbook of
  Discrete and Computational Geometry}, 3rd~ed., J.~E. Goodman, J.~O'Rourke,
  and C.~D. T\'{o}th, Eds.\hskip 1em plus 0.5em minus 0.4em\relax CRC press,
  2018, ch.~51.

\bibitem{GamETAL18}
J.~D. Gammell, T.~D. Barfoot, and S.~S. Srinivasa, ``Informed sampling for
  asymptotically optimal path planning,'' \emph{{IEEE} Trans. Robotics},
  vol.~34, no.~4, pp. 966--984, 2018.

\bibitem{SalH16}
O.~Salzman and D.~Halperin, ``Asymptotically near-optimal {RRT} for fast,
  high-quality motion planning,'' \emph{IEEE Trans. Robotics}, vol.~32, no.~3,
  pp. 473--483, June 2016.

\bibitem{ArslanT13}
O.~Arslan and P.~Tsiotras, ``Use of relaxation methods in sampling-based
  algorithms for optimal motion planning,'' in \emph{ICRA}, 2013, pp.
  2421--2428.

\bibitem{Naderi15}
K.~Naderi, J.~Rajam\"{a}ki, and P.~H\"{a}m\"{a}l\"{a}inen, ``{RT-RRT*}: A
  real-time path planning algorithm based on {RRT}*,'' in \emph{Conference on
  Motion in Games}, 2015, pp. 113--118.

\bibitem{OtteF16}
M.~W. Otte and E.~Frazzoli, ``{RRT}\({}^{\mbox{x}}\): Asymptotically optimal
  single-query sampling-based motion planning with quick replanning,'' \emph{I.
  J. Robotics Res.}, vol.~35, no.~7, pp. 797--822, 2016.

\bibitem{DevETAL16}
D.~Devaurs, T.~Sim{\'{e}}on, and J.~Cort{\'{e}}s, ``Optimal path planning in
  complex cost spaces with sampling-based algorithms,'' \emph{{IEEE} Trans.
  Automation Science and Engineering}, vol.~13, no.~2, pp. 415--424, 2016.

\bibitem{JSCP15}
L.~Janson, E.~Schmerling, A.~A. Clark, and M.~Pavone, ``Fast marching tree: {A}
  fast marching sampling-based method for optimal motion planning in many
  dimensions,'' \emph{I. J. Robotics Res.}, vol.~34, no.~7, pp. 883--921, 2015.

\bibitem{ManETAL18}
A.~Mandalika, O.~Salzman, and S.~Srinivasa, ``Lazy receding horizon {A}* for
  efficient path planning in graphs with expensive-to-evaluate edges,'' in
  \emph{ICAPS}, 2018, pp. 476--484.

\bibitem{SolHal17}
K.~Solovey and D.~Halperin, ``Efficient sampling-based bottleneck pathfinding
  over cost maps,'' in \emph{IROS}, 2017, pp. 2003--2009.

\bibitem{GamETAL15}
J.~D. Gammell, S.~S. Srinivasa, and T.~D. Barfoot, ``Batch informed trees
  ({BIT}*): Sampling-based optimal planning via the heuristically guided search
  of implicit random geometric graphs,'' in \emph{ICRA}, 2015, pp. 3067--3074.

\bibitem{SK18}
K.~Solovey and M.~Kleinbort, ``The critical radius in sampling-based motion
  planning,'' in \emph{RSS}, 2018.

\bibitem{SJP_ICRA15}
E.~Schmerling, L.~Janson, and M.~Pavone, ``Optimal sampling-based motion
  planning under differential constraints: The driftless case,'' in
  \emph{ICRA}, 2015, pp. 2368--2375.

\bibitem{SJP15}
------, ``Optimal sampling-based motion planning under differential
  constraints: The drift case with linear affine dynamics,'' in \emph{CDC},
  2015, pp. 2574--2581.

\bibitem{KarFra10}
S.~Karaman and E.~Frazzoli, ``Optimal kinodynamic motion planning using
  incremental sampling-based methods,'' in \emph{CDC}, 2010, pp. 7681--7687.

\bibitem{KarFra13}
------, ``Sampling-based optimal motion planning for non-holonomic dynamical
  systems,'' in \emph{ICRA}, 2013, pp. 5041--5047.

\bibitem{WebBer13}
D.~J. Webb and J.~P. van~den Berg, ``Kinodynamic {RRT}*: Asymptotically optimal
  motion planning for robots with linear dynamics,'' in \emph{ICRA}, 2013, pp.
  5054--5061.

\bibitem{PerETAL12}
A.~Perez, R.~Platt, G.~Konidaris, L.~Kaelbling, and T.~Lozano{-}P{\'{e}}rez,
  ``{LQR}-{RRT}*: Optimal sampling-based motion planning with automatically
  derived extension heuristics,'' in \emph{ICRA}, 2012, pp. 2537--2542.

\bibitem{XieETAL15}
C.~Xie, J.~P. van~den Berg, S.~Patil, and P.~Abbeel, ``Toward asymptotically
  optimal motion planning for kinodynamic systems using a two-point boundary
  value problem solver,'' in \emph{ICRA}, 2015, pp. 4187--4194.

\bibitem{GorETAL13}
G.~Goretkin, A.~Perez, R.~Platt, and G.~Konidaris, ``Optimal sampling-based
  planning for linear-quadratic kinodynamic systems,'' in \emph{ICRA}, 2013,
  pp. 2429--2436.

\bibitem{LiETAL16}
Y.~Li, Z.~Littlefield, and K.~E. Bekris, ``Asymptotically optimal
  sampling-based kinodynamic planning,'' \emph{I. J. Robotics Res.}, vol.~35,
  no.~5, pp. 528--564, 2016.

\bibitem{HauserZ16}
K.~Hauser and Y.~Zhou, ``Asymptotically optimal planning by feasible
  kinodynamic planning in a state-cost space,'' \emph{{IEEE} Trans. Robotics},
  vol.~32, no.~6, pp. 1431--1443, 2016.

\bibitem{KunzS14b}
T.~Kunz and M.~Stilman, ``Probabilistically complete kinodynamic planning for
  robot manipulators with acceleration limits,'' in \emph{IROS}, 2014, pp.
  3713--3719.

\bibitem{PapaETAL14}
G.~Papadopoulos, H.~Kurniawati, and N.~M. Patrikalakis, ``Analysis of
  asymptotically optimal sampling-based motion planning algorithms for
  {L}ipschitz continuous dynamical systems,'' \emph{CoRR}, vol. abs/1405.2872,
  2014.

\bibitem{JanIP18}
L.~Janson, B.~Ichter, and M.~Pavone, ``Deterministic sampling-based motion
  planning: Optimality, complexity, and performance,'' \emph{I. J. Robotics
  Res.}, vol.~37, no.~1, pp. 46--61, 2018.

\end{thebibliography}

\end{document}